\documentclass[pdflatex,sn-mathphys-num]{sn-jnl}

\usepackage{graphicx}%
\usepackage{multirow}%
\usepackage{amsmath,amssymb,amsfonts}%
\usepackage{amsthm}%
\usepackage{mathrsfs}%
\usepackage[title]{appendix}%
\usepackage{xcolor}%
\usepackage{textcomp}%
\usepackage{manyfoot}%
\usepackage{booktabs}%
\usepackage{algorithm}%
\usepackage{algorithmicx}%
\usepackage{algpseudocode}%
\usepackage{listings}%
\usepackage{wrapfig}
\usepackage{hyperref}
\usepackage{soul}
\usepackage[normalem]{ulem}
\usepackage{adjustbox}
\usepackage{mathrsfs}

\theoremstyle{thmstyleone}%
\newtheorem{theorem}{Theorem}
\newtheorem{proposition}[theorem]{Proposition}%
\newtheorem{assumption}{Assumption}
\newtheorem{corollary}{Corollary}
\newtheorem{lemma}{Lemma}

\theoremstyle{thmstyletwo}%

\theoremstyle{thmstylethree}%
\newtheorem{definition}{Definition}%

\begin{document}


\title{Improving Category-Aware Active Domain Adaptation via Uncertainty-Guided Region Coverage}



\author[1]{\fnm{My} \sur{Tra Nguyen Huynh}}\email{mynht@vnu.edu.vn}

\author*[1]{\fnm{Viet Cuong} \sur{Ta}}\email{cuongtv@vnu.edu.vn}

\affil[1]{\orgname{VNU University of Engineering and Technology}, \city{Hanoi}, \postcode{100000}, \country{Vietnam}}


\abstract{ 
Active domain adaptation researches focus on alleviating the annotation cost, which comprises source data and target data labelling, as well as enhancing the performance of unsupervised domain adaptation.
However, optimizing the adaptation model based on the average accuracy can lead to the drop in performance of some difficult categories. 
To overcome this issue, a point-wise selection strategy built upon the influence function can be used to improve domain adaptation performance.
In our work, we first address the sensitive and redundant aspects of point-wise selection, which can lead to overfitting of the acquisition function.
We thus introduce a region-based approach that ensures the improvement of the adaptation model by maximizing a positive manifold coverage.
An uncertainty score is further introduced to the selection mechanism to learn a better decision boundary for the target domain.
In comprehensive experiments, our proposed method surpasses the first category-aware active domain adaptation baseline as well as existing active domain adaptation methods across multiple domain adaptation benchmarks. Our implemenation is available at: \url{https://github.com/nhtm2702/WHO_IS_NEXT}.
} 




\maketitle

\section{Introduction}\label{sec1}

Domain Adaptation (DA) aims to leverage knowledge learned in a source domain to deploy in a related target domain \cite{ganin2016domain}.
In this setting, the oracle is required to annotate the whole large amount of data, which comprises a source dataset and a target dataset \cite{fang2020opensetdomainadaptation}.
To make DA works more practically, the line of active learning is combined with DA to study the cost-effectiveness in transfer between related domains \cite{Prabhu_2021_ICCV, xie2023dirichletbaseduncertaintycalibrationactive}.
Within the frame of DA, active learning-based methods require labels for the whole source dataset and only a subset of the target dataset, which satisfies a particular design strategy to improve the adaptation model performance.
Most existing active DA works, such as CLUE \cite{Prabhu_2021_ICCV}, DUC \cite{xie2023dirichletbaseduncertaintycalibrationactive} assess the effectiveness of the adaptation model through the average performance across the entire target category, which can lead to a bias phenomenon.
Therefore, the learner tends to query the easier categories, while the most challenging categories seldom receive the advancements or even suffer the performance decreases after active learning rounds.

To alleviate this issue, CAADA \cite{pmlr-v235-xiao24b} introduces a new active domain adaptation setting that trains a model for a specific category to ensure the improvement of individual categories.  
This work adapts the influence function \cite{Cook1980CharacterizationsOA} to the selection strategy to select the samples that are considered potential to improve each category model's performance.
However, in the early rounds when the model performance is still unreliable, directly using the parameter of the logistic regression model in the point-wise selection strategy may necessitate careful design.
Moreover, when using point-based selection, high influence samples may concentrate in the same local region of the feature space.
This leads to redundant selections, where queried samples are mutually close and provide overlapping information, while the learned model is expected to approximate the underlying data distribution.
Thus, this approach does not guarantee that it will enhance model generalization, and it may lead to overfitting of the acquisition function to the empirical training objective \cite{farquhar2021statisticalbiasactivelearning, Zhan2022AsymptoticOF}. 

In our work, we first discuss the redundant and parameter-dependent issues of point-based scores in the category-aware domain adaptation setting.
Based on the analysis, we leverage the point-based selection to a region-based selection with a positive manifold score.
The positive manifold score is designed to measure the coverability of candidate samples to the underlying target distribution through Monte-Carlo approximation.
Moreover, to learn a better decision boundary, we employ the uncertainty guidance, which encourages the selection strategy to query the sample to discover more about the decision boundary regions while reducing focus on regions that are quite certain.
We evaluate our proposed method with the active domain adaptation baselines across three benchmarks.
The results highlight that our method enhances the average accuracy significantly with better underlying data coverages.

\section{Background and Preliminaries}\label{sec2}
\subsection{Related Works}

Domain Adaptation methods are designed to within a transfer learning setting, that operates knowledge generalization on a source domain and then transfers this to a related target domain. 
According to the assumption of target data accessibility, DA settings fall into unsupervised \cite{JMLR:v17:15-239, 9535218}, semi-supervised \cite{Saito_2019_ICCV, basak2025semidavilsemisuperviseddomainadaptation}, weakly-supervised \cite{okuo2025weaklysuperviseddomainadaptationproportionconstrained, doan2024weaklysupervisedtesttimedomain}, and active \cite{Prabhu_2021_ICCV, gavves2015activetransferlearningzeroshot, pmlr-v235-xiao24b}.
In this work, we focus on the active branch that leverages the active learning strategy to reduce annotation cost when adapting knowledge to the target domain.

Active Learning is a machine learning branch that aims to reduce the annotation cost through constructing a selection strategy to choose a subset of unlabeled data to annotate.
ISAL \cite{liu2021influence} selects the most positive influence samples on model performance to annotate. 
\cite{yehuda2022activelearningcoveringlens} defines a cover region to select samples in this and assumes this region can maximize the coverage of data distribution.
In real-world applications, AL can be combined with DA to reduce annotation costs while enhancing model performance on the target domain \cite{Prabhu_2021_ICCV, xie2023dirichletbaseduncertaintycalibrationactive, gavves2015activetransferlearningzeroshot, pmlr-v235-xiao24b}.
CLUE \cite{Prabhu_2021_ICCV} integrates both diversity and uncertainty into an uncertainty-weighted clustering framework. 
DUC \cite{xie2023dirichletbaseduncertaintycalibrationactive} further explores the uncertainty miscalibration in DA problems, and employs a Dirichlet-based evidential model to select uncertain and informative samples.

\subsection{Background}

\noindent\textbf{Active Domain Adaptation (ADA).} 
We denote a domain adaptation model $\mathcal{G}$ trained on a labeled source domain $\mathcal{S}$ and need to be adapted to perform well on a related unlabled target domain $\mathcal{T}$. 
All of the samples from the source dataset are annotated while the target dataset is divided into two parts, $\mathcal{T}_L$ is the labeled target pool and $\mathcal{T}_U$ is the unlabeled target pool, where $\mathcal{T}_L \ll \mathcal{T}_U$.
Both domains have the same set of categories $C$. 
We follow the pool-based setting above to define this problem.
In the initial iteration, we randomly select $\mathcal{B}$ samples from $\mathcal{T}$, then set $\mathcal{T}_L^1 = \mathcal{S} \cup \mathcal{B}$ and $\mathcal{T}_U^1 = \mathcal{T} \setminus \mathcal{B}$.  
Next, at each iteration $k$ with $k \in [2,.., K]$, the DA model $\mathcal{G}$ is retrained on a labeled dataset $\mathcal{T}_L^{k-1}$.
Then, a design selection strategy is used to assess the unlabeled target domain $\mathcal{T}_U^{k-1}$ to actively query and select a set of informative samples $\mathcal{B}$ to annotate.
The labeled target pool is updated as $\mathcal{T}_L^{k}=\mathcal{T}_L^{k-1} \cup \mathcal{B}$ and the unlabeled target pool is $\mathcal{T}_U^{k}=\mathcal{T}_U^{k-1} \setminus \mathcal{B}$.

\noindent\textbf{Category-aware Active Domain Adaptation.} 
Let $\mathcal{G}$ be a pretrained domain adaptation feature extractor, which is used to extract embeddings of source and target datasets.
With a specific category $c$, we binarized the training pool so that samples belonging to class $c$ will be aligned to positive and the rest will be negative.
Let $f_c$, which is parameterized by $\theta$, be a binary logistic regression model trained on the domain-invariant representations extracted by $\mathcal{G}$.
At each iteration $k$, an influence score predictor $f_\mathcal{I}$ is trained to follow the Equation:
\begin{equation}
    \mathcal{L}_{f_\mathcal{I}} = \frac{1}{N_{\mathcal{T}_L}} \sum_{x' \in \mathcal{T}_L} \ell_{mse}(f_\mathcal{I}(x'), \mathcal{I}(x'))
\end{equation}
where $\mathcal{I}(x')$ is the influence estimate of labeled data $x'$. This influence score is calculated based on the Hessian matrix for optimal parameter $\theta$:
\begin{equation}
    \mathcal{I}(x')=-\nabla_{\theta}\mathcal{L}_v\mathbf{H}_{\theta}^{-1} \nabla_{\theta} \ell\!\left(f_c(x'), y\right),
\end{equation}

Then, for each unlabeled target sample $x \in \mathcal{T}_U$, the influence score is estimated as $\mathcal{I}(x) = f_\mathcal{I}(x)$. 
These influence scores are ranked in descending order and $\mathcal{B}$ highest samples are selected.
These $\mathcal{B}$ selected samples are considered most beneficial for recognizing the category $c$.
After $K$ active rounds, the accuracy of classifier $f_c$ will be reported as a performance of the domain adaptation task of category $c$.

\noindent\textbf{Kernel Herding.} 
To model the data coverage of the target distribution, we rely on the Kernel Herding formulation \cite{bae2024generalizedcoveragerobustlowbudget}. 
Let $h_c$ be a non-negative function, which is defined as follows: for any $x$, if $||x_1 - x|| \leq ||x_2 - x||$ then $h_c(x, x_1) \geq h_c(x, x_2)$.
In practice, the measurement $h_c$ implicitly relies on the existence of a good embedding space to compute distance, which is correlated with semantic similarity, and where similar points are likely to cluster together in high density regions.

\section{Method}\label{sec3}
\subsection{Motivation}
Category-aware active domain adaptation is a setting focused on improving the performance for the individual category without adversely affecting others.
The existing method \cite{pmlr-v235-xiao24b}, which designs a point-based selection strategy, requires the Hessian value of the logistic regression $f_c$ to compute the influence score for training the influence predictor $f_\mathcal{I}$. 
Thus, the influence-based acquisition relies on the fitted model $f_c$, typically through first-order gradients and inverse-Hessian terms. 
The resulting scores reflect the influence of a sample under the current optimal parameter $\theta$, rather than an intrinsic data importance. 
Consequently, influence scores may become sensitive to minor changes in the model parameters.

\begin{figure}[!htb]
                \centering
                \includegraphics[width=\textwidth]{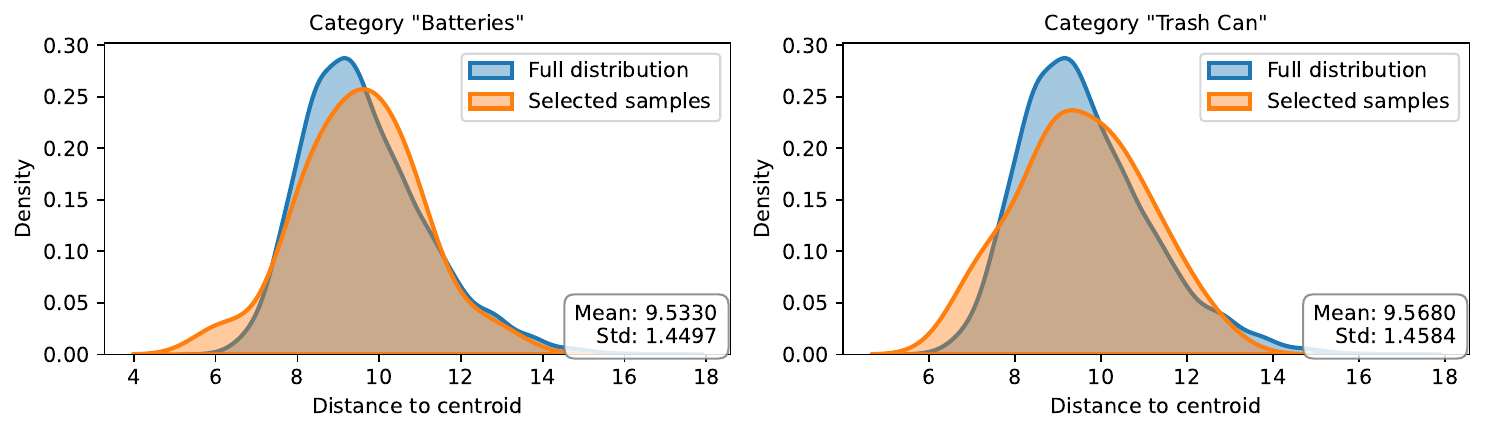}
                \caption{KDE plot for behavior of selection strategy of CAADA method in the $2nd$ active round with two categories, "Batteries" and "Trash Can", in the Ar$\rightarrow$Cl task of the OfficeHome dataset.}
                \label{fig:kde}
            \end{figure} 

On the other hand, when using point-based selection without an explicit diversity or coverage constraint, high influence samples may concentrate in the same local region of the feature space. 
As a result, estimating the model generalization is negatively impacted.
This phenomenon can be seen in the selection strategy of the CAADA method \cite{pmlr-v235-xiao24b}.
We plot the concentration of the selected samples of CAADA in the domain adaptation active learning setting in Figure \ref{fig:kde}.
The samples selected by the point-wise approach of CAADA can be significantly more compact around the category centroid compared to the full distribution.
In both selected categories, the selected samples exhibit higher density peaks and narrower distributions, suggesting that in some cases, this approach mainly focuses on high density regions. 
This effect is particularly clear in the “Trash Can” category, where the selected distribution has a much sharper peak and smaller standard deviation, indicating stronger concentration near the category centroid.

Motivated by these issues, we aim to construct an active strategy, which rarely relies on sensitive parameters of the regression model in the query phase.
Furthermore, our objective is to design a region-based selection strategy, which can eliminate the redundancy of point-based selection as well as boost the adaptation performance for each category.

\begin{figure}[!htb]
                \centering
                \includegraphics[width=\textwidth]{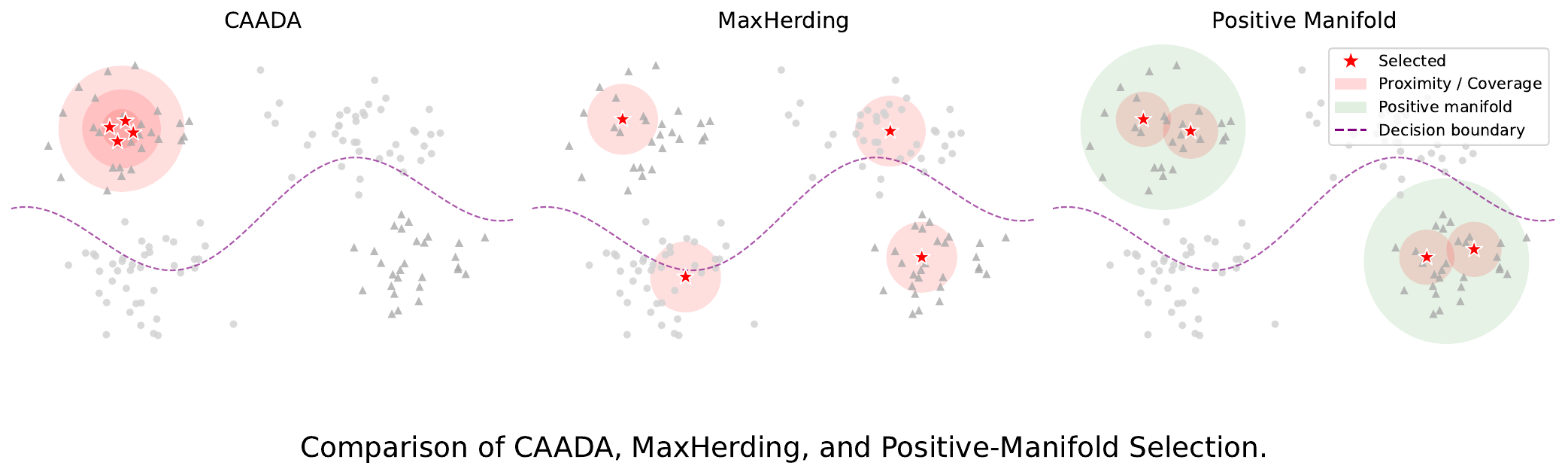}
                \caption{Illustration of different coverage types between CAADA, MaxHerding and our Positive Manifold method.}
                \label{fig:comparison}
            \end{figure} 

In Fig. \ref{fig:comparison}, we visualize differences in assessing coverage on an illustration dataset to demonstrate the behaviour of the selection strategy. 
CAADA (left) focuses on finding high influence samples for a specific category, which can lead to a local concentration.
MaxHerding (middle) introduces a proximity measure for the whole data distribution while ignoring the category-based setting, thus, it may still select samples from negative or task-irrelevant regions.
In contrast, our positive manifold selection (right) performs coverage only over the positive constraint, thus reducing redundant selections and avoiding selection in irrelevant regions.

\subsection{Uncertainty-Guided Selection with Maximizing Region Coverage}
\label{sec4}

To avoid the limitations of the point-wise selection approach, we propose region-based active learning into a domain adaptation framework.
We rely on the Kernel Herding \cite{bae2024generalizedcoveragerobustlowbudget, bae2025uncertaintyherdingactivelearning} idea to introduce the region-based selection strategy.
However, instead of only assessing the coverage of the whole data distribution, our proposed method is extended along two aspects.
First, we estimate category-relevant coverage regions, which correspond to the positive manifold of each category, allowing the selection strategy to focus on category-specific target regions. Second, to ensure that the covered regions are informative for adaptation, we introduce an informative positive manifold coverage, which prioritizes regions where the model lacks confidence.

Our proposed framework is illustrated in Figure \ref{fig:pipeline}.
We use the feature representations of both the source and target domains for training the objective learner $f_c$ as well as match the target distribution to the source distribution with respect to category $c$.
The latter objective can be achieved by optimizing the coverage gain, which is detailed by the kernel-based score $h_c$.
Specifically, we compute a positive manifold coverage worth respect to the category $c$, by relying on $h_c$ and the labeled pool $\mathcal{T}_L$.
Following the region-based approaches in active learning \citep{bae2024generalizedcoveragerobustlowbudget, bae2025uncertaintyherdingactivelearning}, our goal is to approximate the proposed coverage objective using Monte-Carlo estimation over the unlabeled target pool $\mathcal{T}_U$.
However, the data coverage assumption does not directly transfer to a better decision boundary.
Therefore, we further introduce an uncertainty guide process to estimate how positive manifold coverage is accumulated over the unlabeled target samples. 
This design enables the positive manifold coverage strategy to prioritize informative regions while still preserving the category-aware positive manifold objective.

\begin{figure}[!htb]
                \centering
                \includegraphics[width=\textwidth]{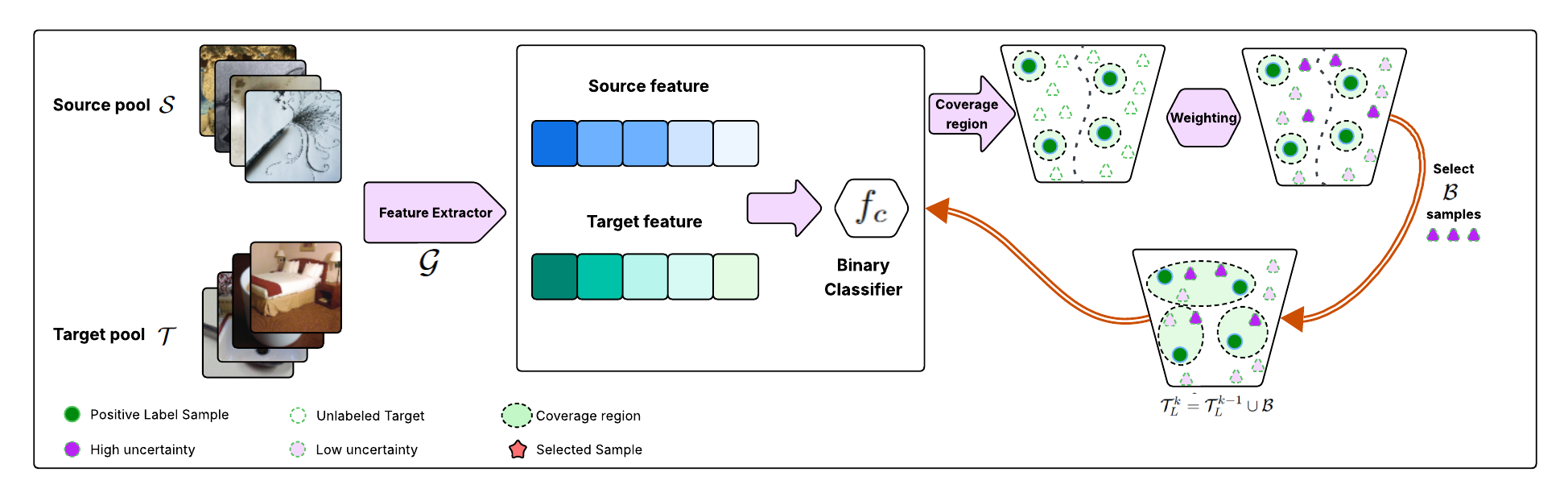}
                \caption{Our proposed category-aware active domain adaptation method focuses on individual categories via uncertainty-Guided Region Coverage.}
                \label{fig:pipeline}
            \end{figure}

With the labeled pool $\mathcal{T}_L$, the data coverage in a category-based setting can be defined on the input feature representation and the kernel function $h_c$.
We denoted $P(\mathcal{T}_L; h_c)$ as the coverage of the labeled pool $\mathcal{T}_L$ compared with the whole data distribution, in the ideal scenario, our $P(\mathcal{T}_L; h_c) \approx 1$.
With the active domain adaptation setting, where the annotation cost is expensive, we want to select a subset of samples from the unlabeled pool $\mathcal{T}_U$ that maximize the $P(\mathcal{T}_L; h_c)$ with a specific budget $\mathcal{B}$.

Consider a classifier model $f_c$ for category $c$, which is trained on a \textit{c-category} labeled dataset. 
The adapted general training loss compared to the true labeling function $\hat f_c$ is defined as: 
 \begin{equation}
     \Pr_{x} \big(\hat f_c(x) \neq f_c(x)\big) \leq  \left(1 - \mathbb{E}_x\left[\max_{x' \in \mathcal{T}_L} h_c(x, x')\right]\right) + \mathbb{E}_x\left[\max_{x': f_c(x) \neq f_c(x')} h_c(x, x')\right]
\label{eq:loss}
\end{equation}

This training error decreases when there exists a candidate sample $x$ which is highly similar to $x'$ and has the same category in the feature space, and less overlap between different classes of $x'$ and $x$. 
To optimize the first expectation on the right side of training error, we define a concept as follows: 
\begin{definition}[\textit{$P-$Coverage function}]
    \label{def:1}
    With $h_c$ is a nonnegative function as denoted above, we define a \textit{coverage gain} term as follows:
    \begin{equation} \label{eq:4}
        \Delta h_c(x; \mathcal{T}_L) = P(\mathcal{T}_L \cup \{x\};h_c) - P(\mathcal{T}_L;h_c),
    \end{equation}
    where $P(\mathcal{T}_L;h_c)=  \mathbb{E}_{x \sim p(x| y=c) }\left[\max_{x' \in \mathcal{T}_L} h_c(x, x')\right]$.
\end{definition}
The above definition is a straightforward extension from region-based selection of \cite{bae2024generalizedcoveragerobustlowbudget, bae2025uncertaintyherdingactivelearning} to the coverage of a specific category $c$. The following lemma states the monotonicity of the coverage function.

\begin{lemma}[Monotonicity]
Let $A \subseteq B$. The coverage function $P(\mathcal{T}_L;h_c)$ is monotone non-decreasing, 
\begin{equation}
P(A;h_c) \le P(B;h_c).
\end{equation}
\label{lem:1}
\end{lemma}

\begin{proof}
For any sample $x$, since $A \subseteq B$, we have
\begin{equation*}
\max_{x' \in A} h_c(x,x')
\le
\max_{x' \in B} h_c(x,x').
\end{equation*}
Taking expectation with respect to $p(x|y=c)$ on both sides,
\begin{equation*}
P(A;h_c)
=
\mathbb{E}_{x \sim p(x|y=c)}
\left[
\max_{x' \in A} h_c(x,x')
\right]
\le
\mathbb{E}_{x \sim p(x|y=c)}
\left[
\max_{x' \in B} h_c(x,x')
\right]
=
P(B;h_c).
\end{equation*}
Therefore, $P(\mathcal{T}_L;h_c)$ is monotone non-decreasing.
\end{proof}

Built upon the coverage gain definition, our objective is to quantify the utility of a queried sample by the increase in category coverage gain.
To compute the coverage $P(\mathcal{T}_L;h_c)$ term, we use Monte-Carlo approximation: 
\begin{equation}
\begin{aligned}
    P(\mathcal{T}_L;h_c)=  \mathbb{E}_{x \sim p(x| y=c) }\left[\max_{x' \in \mathcal{T}_L} h_c(x, x')\right] \\ 
    \approx \frac{1}{N_{\mathcal{T}_U}} \sum_{i=1}^{N_{\mathcal{T}_U}} \max_{\mathbf{x}'\in \mathcal{T}_L} h_c(x_i,x')
\end{aligned}
\label{eq:coverage}
\end{equation}

Then, we can rewrite the coverage gain in Eq. \ref{eq:4} in detail as:
\begin{equation}
\begin{aligned}
    \Delta h_c(x; \mathcal{T}_L) 
    \approx \frac{1}{N_{\mathcal{T}_U}} \sum_{i=1}^{N_{\mathcal{T}_U}} \max_{\mathbf{x}'\in \mathcal{T}_L \cup \{x\}} h_c(x_i,x') - \frac{1}{N_{\mathcal{T}_U}} \sum_{i=1}^{N_{\mathcal{T}_U}} \max_{\mathbf{x}'\in \mathcal{T}_L} h_c(x_i,x')
\end{aligned}
\label{eq:coverage_gain}
\end{equation}

Our approach considers each sample $x$ of unlabeled target pool $\mathcal{T}_U$, adds it to $\mathcal{T}_L$ then computes the coverage gain.
After that, we choose the sample that has the maximal $\Delta h_c(x; \mathcal{T}_L)$.
This approach relies on the characteristics of the extracted features rather than on the regression parameters, which are quite sensitive during the training process.

In the category-based setting, the category $c$ can be viewed as a class-conditional location to select the next batches.
This information is integrated into our selection mechanism through the positive manifold score.
The positive manifold score approximates the effect of selecting a specific sample on the coverage of a given class $c$.
We first give a definition of the positive manifold region.
\begin{definition}[Positive Coverage Gain]
\label{def:positiveM}
Let $c$ denote the target category and let
\[
y_c=
\begin{cases}
1, & y=c,\\
0, & y\neq c,
\end{cases}
\]
be the corresponding binarized label. The \emph{positive manifold} associated with category $c$ is defined as the support of the conditional distribution
\[
p(x\mid y_c=1),
\]
which consists of all labeled samples belonging to category $c$. We denote this manifold by
\[
\mathcal M_c^{+} =
\{x \in \mathcal{T}_L : y_c=1\}.
\]

The coverage gain of a sample $x$ is defined as:
\begin{equation}
\begin{aligned}
    \Delta h_c(x; \mathcal M_c^{+}) 
    \approx \frac{1}{N_{\mathcal{T}_U}} \sum_{i=1}^{N_{\mathcal{T}_U}} \max_{\mathbf{x}'\in \mathcal M_c^{+} \cup \{x\}} h_c(x_i,x') - \frac{1}{N_{\mathcal{T}_U}} \sum_{i=1}^{N_{\mathcal{T}_U}} \max_{\mathbf{x}'\in \mathcal M_c^{+}} h_c(x_i,x')
\end{aligned}
\label{eq:positive_gain}
\end{equation}

\end{definition}

This equation ensures the positive region coverage naturally discourages selecting redundant samples that are already well represented by the current labeled set.

To construct this positive manifold region, we compute the kernel between all label samples that have a positive category and the whole data pool.
However, in the first iteration, we do not have accessibility to the label of the target pool.
Therefore, we employ a warm-up phase by randomly selecting samples to add in the first labeled pool.
An alternative choice is to utilize only positive labeled samples from the source domain.
However, the resulting positive manifold estimation may result in bias to the source domain, which diverges from the target data distribution if the target and the source domains are distant from each other.


Although positive manifold coverage encourages the expansion of data distribution to reduce redundancy, it implicitly assumes that the informativeness of all regions of the positive manifold is equal. 
In particular, samples for which the current category classifier is uncertain may provide more useful supervision than samples that are already classified with high confidence.
Therefore, we use the predictive uncertainty of the category classifier as an informativeness guide so that regions near the current category decision boundary receive larger emphasis while the overall coverage structure is still retained.
The uncertainty measure is defined through the entropy of a predicted probability as:
\begin{equation}
    \mathcal{U}(x;f_c) = -p \log_2p - (1-p)\log_2(1-p)
\label{eq:entropy}
\end{equation}
with $p = P(y=1|x)$ is the predicted probability of sample $x$.
The positive manifold score in Eq. \ref{eq:positive_gain} is used to combine with the uncertainty measure
\begin{equation}
     s(x, \mathcal M_c^{+}; f_c) =  \frac{1}{N_{\mathcal{T}_U}} \sum_{i=1}^{N_{\mathcal{T}_U}} \mathcal{U}(x_i;f_c) \left(\max_{\mathbf{x}'\in \mathcal M_c^{+} \cup \{x\}} h_c(x_i,x') - \max_{\mathbf{x}'\in \mathcal M_c^{+}} h_c(x_i,x')\right)
\label{eq:score}
\end{equation}
The resulting score $s(x, \mathcal M_c^{+}; f_c)$ combines the uncertainty of both the point under consideration and the region it covers, rather than just the point itself.
Note that, although the uncertainty signal $\mathcal{U}(x;f_c)$ is a sample-based estimation, it is not used to select a new sample directly.
Instead, it serves as a guidance signal for estimating category-aware coverage gain on the positive manifold.

Our framework is summarized in Algorithm \ref{alg:hgt}.
The overall setup follows the general setting of category-aware active domain adaptation \cite{pmlr-v235-xiao24b}.
After the warming up phase, for each iteration, our positive manifold coverage method starts by constructing the kernel function $h_c$.
Then, the coverage gain and uncertainty score of each sample are estimated by employing Eq. \ref{eq:coverage} and Eq. \ref{eq:entropy}.
The new sample $x_i$ is then selected to maximize the score 
$s(x, \mathcal M_c^{+}; f_c)$ in Eq. \ref{eq:score}.


For the complexity estimation, the time complexity of initial $h_c$ is $O(n_Ln_{\mathcal{T}_U})$, where $n_L$ is the number of labeled samples and $n_{\mathcal{T}_U}$ is the number of unlabeled samples.
The uncertainty scores for all unlabeled samples can be computed in $O(n_{\mathcal{T}_U})$, excluding the model inference cost.
Computing the $s(x, \mathcal M_c^{+}; f_c)$ score will take $O(n_{\mathcal{T}_U}^2)$.
The $\arg \max$ operation takes $O(n_{\mathcal{T}_U})$.
This selection strategy loops for $\mathcal{B}$ rounds, so the overall time complexity is $O(\mathcal{B}n_{\mathcal{T}_U}^2)$.
Consequently, the overall complexity of our proposed method is $O(n_Ln_{\mathcal{T}_U} + \mathcal{B}n_{\mathcal{T}_U}^2)$.

\begin{algorithm}
\caption{Uncertainty-Aware Region-based Algorithm}\label{alg:hgt}
\begin{algorithmic}[1]
\State \textbf{Input:} Labeled source data pool $\mathcal{S}$, unlabeled target data pool $\mathcal{T}$, budget $\mathcal{B}$ per round and total rounds $K$ and a base DA model $\mathcal{G}$.
\State \textbf{Output:} Classification model $f_c$ for category $c$.

\State \textbf{for} $c = 1$ to $C$ do

\State \hspace{1em} Binarized the labeled data representations.
\State \hspace{1em} Randomly select $\mathcal{B}$ samples, as demonstrated in Section \ref{sec3}.
\State \hspace{1em} Initial the labeled pool is $\mathcal{T}_L^1 = \mathcal{S} \cup \mathcal{B}$ and the unlabeled pool is $\mathcal{T}_U^1 = \mathcal{T} \setminus \mathcal{B}$.
\State \hspace{1em} Train an initial classification model $f'$ with the $\mathcal{T}_L^1$ pool.

\State \hspace{1em} \textbf{for} $k = 2$ to $K$ do
\State \hspace{2em} $\mathcal{S}_B = \emptyset$ is a temporary set of $\mathcal{B}$ selected samples in each active round.
\State \hspace{2em} Compute $h_c \in \mathbb{R}^{|\mathcal{T}_U^{k-1}|}$, where $h_c^i = \max_{x' \in \mathcal{M}_c^{+}} h_c(x_i, x')$, and $x_i \in \mathcal{T}_U^{k-1}$.
\State \hspace{2em} \textbf{for} $b = 1$ to $\mathcal{B}$ do
\State \hspace{3em} Calculate coverage gain for each candidate sample $\Delta h_c(x_i; \mathcal M_c^{+})$ with Eq. \ref{eq:coverage}
\State \hspace{3em} Calculate uncertainty score for each sample in $\mathcal{T}_U^{k-1}$ with Eq. \ref{eq:entropy}.
\State \hspace{3em} Select $x^* = \arg\max_{x\in \mathcal{T}_U^{k-1}} s(x, \mathcal M_c^{+}; f_c) $ with $s(x, \mathcal M_c^{+}; f_c)$ is computed following the Eq. \ref{eq:score}.
\State \hspace{3em} Update  $\mathcal{S}_B = \mathcal{S}_B \cup \{x^*\}$
\State \hspace{3em} Update $h_c^i \leftarrow \max (h_c(x_i,x^*),h_c^i)$.
\State \hspace{2em} \textbf{end for}
\State \hspace{2em} Update $\mathcal{T}_U^{k} = \mathcal{T}_U^{k-1} \setminus \mathcal{S}_B$
\State \hspace{2em} Update  $\mathcal{T}_L^{k} = \mathcal{T}_L^{k-1} \cup \mathcal{S}_B$
\State \hspace{1em} \textbf{end for}
\State \hspace{1em} Retrain classification model $f'$ with $\mathcal{T}_L^{K}$.
\State \hspace{1em} Set classification model $f_c = f'$.
\State \textbf{end for}
\end{algorithmic}
\end{algorithm}
\subsection{Theory Analysis}
In this section, we analyze the point-based behavior in a category-aware setting and compare it with our proposed region-based selection mechanism.
The analysis provides a guarantee on reducing the redundant selection with our proposed method.
We first assume a Lipschitz Continuity on the estimation of Category-based Influence Score:

\begin{assumption}[Lipschitz Continuity of Category-based Influence]\label{ass:1}
For a target category $c$, the category-wise influence score $I_c(x)$ is $L_I$-Lipschitz continuous in the feature space. That is, for any two target samples $x_i, x_j \in \mathcal{T}_U$,
\[
|I_c(x_i) - I_c(x_j)|
\le
L_I \, d(g_i, g_j),
\]
where $\mathcal{T}_U$ is unlabeled target pool, 
$g_i=\mathcal{G}(x_i)$ and 
$g_j=\mathcal{G}(x_j)$ are features extracted from base DA model $\mathcal{G}$ and $d(\cdot,\cdot)$ denotes a distance metric in the feature space.
\end{assumption}

From this assumption, we present a proposition about the local concentration induced by Lipschitz in case of influence-based selection:

\begin{proposition}[Lipschitz-induced Local Influence Concentration]\label{prop:2}
Let \(x^\star \in \mathcal{T}_U\) be a high influence sample with \(I_c(x^\star)=M\). For any radius \(\lambda>0\), define the local neighborhood $r_\lambda(x^\star)=\{x \in \mathcal{T}_U: d(g, g^\star) \le \lambda\}$.
If $r_\lambda(x^\star)$ contains at least $B$ samples and 
 $M - L_I\lambda  > \max_{z \in \mathcal{T}_U\setminus r_\lambda(x^\star)} I_c(z)$,
then the influence top-$B$ selection satisfies
\[
S_B \subseteq r_\lambda(x^\star).
\]
\end{proposition}

\begin{corollary}[Maximum distance of top-B selected samples]\label{coro:2}
For any pair of $(x_i, x_j)$ in the top-B selection that is located in any radius $\lambda$,
\[
\max_{x_i,x_j \in S_B} d(g_i,g_j) \le 2\lambda.
\]
\end{corollary}

\begin{proof}
By the Lipschitz continuity of \(I_c\), for any \(x \in r_\lambda(x^\star)\),
\[
|I_c(x)-I_c(x^\star)| \le L_I d(g,g^\star) \le L_I \lambda.
\]
Therefore,
\[
I_c(x) \ge I_c(x^\star)-L_I \lambda = M-L_I \lambda.
\]
If \(M-L_I \lambda\) is larger than the maximum influence score outside \(r_\lambda(x^\star)\), and \(r_\lambda(x^\star)\) contains at least \(B\) samples, then the top-\(B\) samples ranked by \(I_c(x)\) must all lie inside \(r_\lambda(x^\star)\). Hence \(S_B \subseteq r_\lambda(x^\star)\). Since every selected sample lies within radius \(\lambda\) of \(x^\star\), by the triangle inequality, for any \(x_i,x_j \in S_B\),
\[
d(g_i,g_j) \le d(g_i,g^\star)+d(g^\star,g_j) \le 2\lambda.
\]
\end{proof}

Therefore, Lipschitz continuity provides a theoretical explanation for local influence concentration.
If influence score changes smoothly across features, then high influence samples will attract close ones as well. 
If that neighbor is densely populated and has a higher score than the outside, the influence-only top-$\mathcal{B}$ will choose many points close together.
To solve this issue, we propose a region-based approach that encourages the selection strategy to choose a sample which can approximate the underlying data distribution.
We formalize this intuition with the following proposition.

\begin{proposition}[Redundancy reduction]\label{prop:3}
Consider an unlabeled candidate $x$, if $x$ is redundant with respect to the current selected set, 
\[
h_c(x_i,x)\le \max_{x'\in \mathcal M_c^{+}}h_c(x_i,x')
\quad
\text{for most } x_i,
\]
then
\[
\Delta h_c(x; \mathcal M_c^{+}) \approx 0.
\]
\end{proposition}

\begin{proof}
By Eq. \ref{eq:positive_gain}, we have
\[
\Delta h_c(x; \mathcal M_c^{+})
\approx
\frac{1}{N_{\mathcal{T}_U}}\sum_{i=1}^{N_{\mathcal{T}_U}}
\left[
\max_{x'\in \mathcal M_c^{+}\cup\{x\}}h_c(x_i,x')
-
\max_{x'\in \mathcal M_c^{+}}h_c(x_i,x')
\right].
\]
Let $\mathcal{H} (\mathcal M_c^{+}) = \max_{x'\in \mathcal M_c^{+}}h_c(x_i,x')$. Then, $\max_{x'\in  \mathcal M_c^{+} \cup\{x\}}h_c(x_i,x')$ can be expressed as follows:
\[
\max_{x'\in  \mathcal M_c^{+} \cup\{x\}}h_c(x_i,x')
=
\max\left\{ \mathcal{H} (\mathcal M_c^{+}),h_c(x_i,x)\right\},
\]
We have
\[
\Delta h_c(x;\mathcal M_c^{+}) 
\approx
\frac{1}{N_{\mathcal{T}_U}}\sum_{i=1}^{N_{\mathcal{T}_U}}
\left[
\max\left\{\mathcal{H} (\mathcal M_c^{+}),h_c(x_i,x)\right\}-\mathcal{H} (\mathcal M_c^{+})
\right],
\]
This is equivalent to
\[
\Delta h_c(x; \mathcal M_c^{+}) 
\approx
\begin{cases}
        \frac{1}{N_{\mathcal{T}_U}}\sum_{i=1}^{N_{\mathcal{T}_U}}
\left[
h_c(x_i,x)-\mathcal{H}(\mathcal M_c^{+})
\right], & \text{ if } \mathcal{H}(\mathcal M_c^{+}) \leq h_c(x_i,x), \\
    0 & \text{ otherwise}.
\end{cases}
\]
We can rewrite it in the form of
\[
\Delta h_c(x; \mathcal M_c^{+}) 
\approx 
\frac{1}{N_{\mathcal{T}_U}}\sum_{i=1}^{N_{\mathcal{T}_U}}
\left[
h_c(x_i,x)-\mathcal{H} (\mathcal M_c^{+})
\right]^+
\]
with $[\cdot]^+$ denotes the positive-part operator, defined as $[z]^+ = \max(z,0)$.

Thus, if $h_c(x_i,x)\le \max_{x'\in \mathcal M_c^{+}}h_c(x_i,x')$ for most $x_i$, the corresponding terms are zero or small.
\end{proof}

While our region-based approach selects the candidate sample $x$, which has the maximum coverage gain $\Delta h_c(x; \mathcal M_c^{+})$.
This constraint can reduce the redundant sample that is learned by the current model, in other words, the sample that the current model covers. 

\begin{corollary}[Reduction in returns]\label{coro:3}
For any \(A\subseteq B\) and candidate \(x\notin B\),
\[
\Delta h_c(x; A) \ge \Delta h_c(x; B).
\]
\end{corollary}
\begin{proof} 
Since the coverage function is monotone non-decsreasing (Lemma \ref{lem:1}), we have
$\mathcal{H}(A) \leq \mathcal{H}(B)$, with $\mathcal{H}(A), \mathcal{H}(B)$ are computed as $\mathcal{H}(\mathcal M_c^{+})$ in the Proposition \ref{prop:3}. Thus, 
\[
h_c(x_i,x)-\mathcal{H}(A)
\ge
h_c(x_i,x)-\mathcal{H}(B).
\]

Since the positive-part operator is monotone, we have
\[
\left[
h_c(x_i,x)-\mathcal{H}(A)
\right]^+
\ge
\left[
h_c(x_i,x)-\mathcal{H}(B)
\right]^+.
\]

Summing over all $x_i$ samples and dividing by $N_{\mathcal{T}_U}$, we obtain
\[
\Delta h_c(x;A)
\ge
\Delta h_c(x;B).
\]
\end{proof}

In our framework, we incorporate the uncertainty estimation of the candidate region as in Eq. \ref{eq:score}. 
Therefore, it is required that the uncertainty weight is able to preserve the positive region coverage.
To prove that, we first assume a uniform distribution on the estimation of uncertainty $\mathcal{U}$:

\begin{assumption}[Calibrated uncertainty signal]\label{ass:3}
Let $\alpha \in (0,1)$, the uncertainty weights are uniformly bounded:
\[
0 < \alpha \le \mathcal{U}(x_i; f_c) \le 1,
\quad
\forall i\in\{1,\ldots,N_{\mathcal{T}_U}\}.
\]
\end{assumption}

\begin{theorem}[Bounded uncertainty preserves coverage gain]\label{theo:5}
For any labeled set $\mathcal M_c^{+}$ and any candidate $x \in \mathcal M_c^{+}$, the uncertainty-weighted coverage gain satisfies
\[
\alpha \Delta h_c(x; \mathcal M_c^{+})
\le
s(x, \mathcal M_c^{+}; f_c)
\le
\Delta h_c(x; \mathcal M_c^{+}).
\]
Therefore, uncertainty weighting preserves the positive region coverage gain up to multiplicative constants.
\end{theorem}

\begin{proof}
We rewrite Eq. \ref{eq:score} as the form of Eq. \ref{eq:positive_gain} in Proposition \ref{prop:3},
\[
s(x, \mathcal M_c^{+}; f_c)
=
\frac{1}{N_{\mathcal{T}_U}}
\sum_{i=1}^{N_{\mathcal{T}_U}}
\mathcal{U}(x_i; f_c)
\left[
h_c(x_i,x)-\mathcal{H} (\mathcal M_c^{+})
\right]^+.
\].

Since $0<\alpha\le \mathcal{U}(x_i; f_c) \le 1$,
we have
\[
\alpha
\left[
h_c(x_i,x)-\mathcal{H}(\mathcal M_c^{+})
\right]^+
\le
\mathcal{U}(x_i; f_c)
\left[
h_c(x_i,x)-\mathcal{H}(\mathcal M_c^{+})
\right]^+
\le
\left[
h_c(x_i,x)-\mathcal{H}(\mathcal M_c^{+})
\right]^+.
\]

Summing over all $x_i$ samples and dividing by $N_{\mathcal{T}_U}$, we obtain
\[
\alpha \Delta h_c(x; \mathcal M_c^{+})
\le
s(x, \mathcal M_c^{+}; f_c)
\le
\Delta h_c(x; \mathcal M_c^{+}).
\]
\end{proof}

In CAADA \cite{pmlr-v235-xiao24b}, the authors provide category awareness through a pointwise influence predictor but do not explicitly model redundancy among selected samples.
MaxHerding \cite{bae2024generalizedcoveragerobustlowbudget} gives a framework which can reduce region-level redundancy through generalized coverage.
MaxHerding's coverage is defined over the entire data distribution, without distinguishing whether a sample contributes to the coverage of a particular category. 
Thus, these selected samples are class-agnotic.
Our method bridges these two views: it inherits the region-aware coverage principle of MaxHerding, while introducing the category-based setting in active domain adaptation.
Our novelty formulation restricts the coverage measure to the positive region of the target class and then guides the selection strategy by the uncertainty signal. 
The uncertainty-weighted selection principle in our approach is common in active learning, and has been shown to be more effective than using only one of the two signals individually in active domain adaptation \cite{Prabhu_2021_ICCV}. 
For candidates with comparable category relevance, the proposed objective prefers the candidate that yields a larger positive region coverage gain.
Hence, it preserves the category-aware motivation while reducing the redundancy caused by independent top-$\mathcal{B}$ pointwise selection.

\section{Experiments}\label{sec5}
\subsection{Settings}

To evaluate the efficiency of our proposed method, we compare it with five active domain adaptation baselines and a category-aware active domain adaptation, which include Random, DANN \cite{ganin2016domain}, CLUE \cite{Prabhu_2021_ICCV}, DUC \cite{xie2023dirichletbaseduncertaintycalibrationactive}, ISAL \cite{liu2021influence}, CAADA \cite{pmlr-v235-xiao24b}.


The experiments are carried out on three popular benchmarks in the domain adaptation field: Office-Home\cite{venkateswara2017deep}, DomainNet-126 \cite{peng2019moment} and VisDa-2017 \cite{peng2017visda}:
\begin{itemize}
\item{Office-Home \cite{venkateswara2017deep}}: 
comprises four domains: Real-World, Product, Clip Art, and Artistic. Each domain includes images of 65 categories.

\item{DomainNet-126 \cite{peng2019moment}}: 
We follow the same protocol as \cite{pmlr-v235-xiao24b} to choose 126 categories and six-directions from four domains: Real, Clipart, Painting, and Sketch.

\item{VisDa-2017 \cite{peng2017visda}}: 
is separated into training, validation, and testing domains and comprises 
12 categories.
\end{itemize}

We use the ResNet-50 \cite{he2015deepresiduallearningimage} is trained from scratch as the backbone feature extractor and a DANN \cite{ganin2016domain} model is trained as the base DA method following CLUE \cite{Prabhu_2021_ICCV}. 
The feature representations for both the labeled source domain and the unlabeled target domain are extracted from a fixed feature extractor pre-trained by the base DA method.
For each category $c$, we binarize the source pool and train the initial model $f'$ with an $L_2$ weight of $1e-3$.
We perform active learning in five iterations. 
With Office-Home, we set $\mathcal{B} = 1\%$ of the total target dataset, and $\mathcal{B}=100$ for DomainNet-126 and VisDa-2017.
For the kernel function, we use the Gaussian kernel $h_c(x,x') = \exp (\frac{1}{2\sigma^2}||x-x'||^2)$, with the batch size is $512$, and the kernel width $\sigma = 0.7$.
We apply the same training and testing protocol in the experiments for baseline methods, and then report the average per-category accuracy.

\subsection{Main results}

We report
the average accuracy per category in Table \ref{tab:office_home} and Table \ref{tab:domainnet_visda}.
As shown in these two tables, our proposed method achieved the highest average accuracy per category in all tasks of three benchmarks, obtaining an average accuracy of $87.73\%$ on Office-Home, $83.93\%$ on DomainNet-126, and $79.37\%$ on VisDA-2017.

In the Office-Home dataset, our approach outperforms the closest setting, CAADA, in all tasks with the average difference is $8.41\%$.
While CAADA performs a competitive improvement with the other baselines, our proposed method shows notable superiority with the average difference is $26.16\%$ in the DomainNet-126 dataset.
With the adaptation from Synthetic-to-Real domain of the VisDA-2017 dataset, our proposed method outperforms the baselines with $12.69\%$ as well.

\begin{table}[t]
    \centering
    \caption{Average per-category predicting accuracy (\%) for our proposed method on \textit{Office-Home}. \textit{Avg.} column represents the average accuracy across all 6 tasks in the same table.}
    \label{tab:office_home}
    \footnotesize
    \setlength{\tabcolsep}{2pt}
    \begin{tabular}{l|ccccccc}
    \toprule
     & Ar$\rightarrow$Cl & Ar$\rightarrow$Pr & Ar$\rightarrow$Rw & Cl$\rightarrow$Ar & Cl$\rightarrow$Pr & Cl$\rightarrow$Rw & Avg. \\
    \midrule
    {\fontsize{6}{7}\selectfont DANN} & 36.99 $\pm$ 0.52 & 51.83 $\pm$ 0.20 & 58.81 $\pm$ 0.22 & 36.54 $\pm$ 0.12 & 48.59 $\pm$ 0.35 & 59.17 $\pm$ 0.43 & 48.65 \\
    {\fontsize{7.5}{7}\selectfont Random} & 47.19 $\pm$ 0.60 & 55.41 $\pm$ 0.55 & 70.14 $\pm$ 0.34 & 50.92 $\pm$ 0.36 & 65.08 $\pm$ 0.36 & 62.77 $\pm$ 0.33 & 58.58 \\
    {\fontsize{6}{7}\selectfont CLUE} & 48.14 $\pm$ 0.76 & 59.34 $\pm$ 0.73 & 74.34 $\pm$ 0.60 & 54.23 $\pm$ 0.52 & 78.74 $\pm$ 0.89 & 66.92 $\pm$ 0.88 & 63.62 \\
    {\fontsize{6}{7}\selectfont DUC} & 51.23 $\pm$ 0.34 & 61.81 $\pm$ 0.12 & 75.31 $\pm$ 0.23 & 56.41 $\pm$ 0.33 & 79.82 $\pm$ 0.24 & 70.02 $\pm$ 0.44 & 65.76 \\
    {\fontsize{6}{7}\selectfont ISAL} & 64.84 $\pm$ 0.22 & 79.23 $\pm$ 0.12 & 78.66 $\pm$ 0.32 & 62.93 $\pm$ 0.54 & 81.24 $\pm$ 0.55 & 75.62 $\pm$ 0.32 & 73.75 \\
    {\fontsize{6}{7}\selectfont CAADA} & 69.17 $\pm$ 0.12 & 88.42 $\pm$ 0.10 & 86.25 $\pm$ 0.22 & 68.89 $\pm$ 0.13 & 84.47 $\pm$ 0.15& 82.74 $\pm$ 0.43 & 81.01 \\
    {\fontsize{7.5}{7}\selectfont Ours} & \textbf{74.23 $\pm$ 0.10}& \textbf{94.08 $\pm$ 0.05 }& \textbf{94.01 $\pm$ 0.06 }& \textbf{80.10 $\pm$ 0.12 }& \textbf{90.77 $\pm$ 0.09 }& \textbf{89.19 $\pm$ 0.19}& \textbf{87.06} \\
    \midrule
    \midrule
     & Pr$\rightarrow$Ar & Pr$\rightarrow$Cl & Pr$\rightarrow$Rw & Rw$\rightarrow$Ar & Rw$\rightarrow$Cl & Rw$\rightarrow$Pr & Avg. \\
    \midrule
    {\fontsize{6}{7}\selectfont DANN} & 32.83 $\pm$ 0.47 & 32.98 $\pm$ 0.32 & 59.26 $\pm$ 0.33 & 52.28 $\pm$ 0.27 & 68.93 $\pm$ 0.29 & 45.36 $\pm$ 0.36 & 48.61 \\
    {\fontsize{7.5}{7}\selectfont Random} & 48.41 $\pm$ 0.22 & 45.77 $\pm$ 0.27 & 72.11 $\pm$ 0.18 & 61.92 $\pm$ 0.26 & 77.61 $\pm$ 0.35 & 55.78 $\pm$ 0.13 & 60.26 \\
    {\fontsize{6}{7}\selectfont CLUE} & 53.17 $\pm$ 0.45 & 59.83 $\pm$ 0.46 & 75.13 $\pm$ 0.39 & 65.17 $\pm$ 0.54 & 80.42 $\pm$ 0.32 & 61.14 $\pm$ 0.47 & 65.81 \\
    {\fontsize{6}{7}\selectfont DUC} & 59.23 $\pm$ 0.13 & 60.73 $\pm$ 0.16 & 77.34 $\pm$ 0.27 & 67.35 $\pm$ 0.33 & 79.68 $\pm$ 0.12 & 65.02 $\pm$ 0.30 & 68.22 \\
    {\fontsize{6}{7}\selectfont ISAL} & 64.17 $\pm$ 0.23 & 65.02 $\pm$ 0.25 & 83.51 $\pm$ 0.32 & 71.42 $\pm$ 0.35 & 80.41 $\pm$ 0.36 & 71.51 $\pm$ 0.25 & 72.67 \\
    {\fontsize{6}{7}\selectfont CAADA} & {69.75} $\pm$ 0.15 & {67.85} $\pm$ 0.17 & {88.31} $\pm$ 0.12 & {79.61} $\pm$ 0.21 & {92.11} $\pm$ 0.25 & {74.22} $\pm$ 0.35 & {78.64} \\
    {\fontsize{7.5}{7}\selectfont Ours} & \textbf{85.70 $\pm$ 0.12}& \textbf{74.59 $\pm$ 0.18} &  \textbf{93.87 $\pm$ 0.08} & \textbf{87.28 $\pm$ 0.14} & \textbf{94.10 $\pm$ 0.09}& \textbf{94.98 $\pm$ 0.13}& \textbf{88.39} \\
    \bottomrule
    \end{tabular}
\end{table}

\begin{table}[t]
\centering
\caption{Average per-category predicting accuracy (\%) for Category-Aware Active Domain Adaptation on \textit{DomainNet-126} and \textit{VisDA-2017}. \textit{Avg.} column represents the average accuracy across all 6 tasks of \textit{DomainNet-126} in this table.}
\label{tab:domainnet_visda}
\footnotesize
\fontsize{5.2}{7}\selectfont
\setlength{\tabcolsep}{1.5pt}
\begin{tabular}{l|cccccccc}
\toprule
 & \multicolumn{7}{c}{\textit{DomainNet-126}} & \textit{VisDA-2017} \\
\cmidrule(lr){2-8} \cmidrule(lr){9-9}
 & cl$\rightarrow$pt & cl$\rightarrow$rl & cl$\rightarrow$sk & sk$\rightarrow$cl & sk$\rightarrow$pt & sk$\rightarrow$rl & Avg. & S$\rightarrow$R \\
\midrule
{\fontsize{6}{7}\selectfont DANN} & 34.18 $\pm$ 0.37 & 42.15 $\pm$ 0.32 & 36.12 $\pm$ 0.29 & 31.82 $\pm$ 0.41 & 56.59 $\pm$ 0.38 & 47.24 $\pm$ 0.37 & 41.35& 47.84 $\pm$ 0.28 \\
\midrule
{\fontsize{7.5}{7}\selectfont Random} & 36.25 $\pm$ 0.54 & 44.75 $\pm$ 0.43 & 40.74 $\pm$ 0.49 & 34.41 $\pm$ 0.53 & 58.31 $\pm$ 0.39 & 49.53 $\pm$ 0.31 & 44.00 & 54.19 $\pm$ 0.33 \\
{\fontsize{6}{7}\selectfont CLUE} & 38.06 $\pm$ 0.34 & 45.34 $\pm$ 0.36 & 42.28 $\pm$ 0.46 & 35.63 $\pm$ 0.48 & 60.12 $\pm$ 0.44 & 57.37 $\pm$ 0.43 & 46.47 & 58.54 $\pm$ 0.36\\
{\fontsize{6}{7}\selectfont DUC} & 37.53 $\pm$ 0.37 & 48.91 $\pm$ 0.33 & 39.84 $\pm$ 0.29 & 36.12 $\pm$ 0.27 & 61.35 $\pm$ 0.33 & 58.26 $\pm$ 0.39 & 47.02 & 60.12 $\pm$ 0.40\\
{\fontsize{6}{7}\selectfont ISAL} & 45.05 $\pm$ 0.19& 56.94 $\pm$ 0.21 & 51.31 $\pm$ 0.16 & 40.34 $\pm$ 0.26 & 62.27 $\pm$ 0.30 & 61.64 $\pm$ 0.24  & 52.92 & 62.31 $\pm$ 0.17\\
{\fontsize{6}{7}\selectfont CAADA} & {52.23} $\pm$ 0.13 & {64.21} $\pm$ 0.14 & {56.34} $\pm$ 0.16 & {42.49} $\pm$ 0.20 & {67.94} $\pm$ 0.19 & {63.41} $\pm$ 0.21 & {57.78} & {66.68} $\pm$ 0.23 \\
{\fontsize{8}{7}\selectfont Ours} & \textbf{79.53 $\pm$ 0.14} & \textbf{86.85 $\pm$ 0.17} & \textbf{81.13 $\pm$ 0.13} & \textbf{88.17$\pm$ 0.19} & \textbf{81.57 $\pm$ 0.20} & \textbf{86.30 $\pm$ 0.17} & \textbf{83.93} & \textbf{79.37 $\pm$ 0.18} \\
\bottomrule
\end{tabular}
\end{table}

The improvement observed across different source-target pairs indicates that the advantage is not restricted to a particular adaptation task. 
Moreover, the results observed on all benchmarks are consistent with our motivation that positive-manifold coverage can help identify more informative samples within the target category of interest, thereby making more effective use of the limited annotation budget.

\subsection{Ablation study}
\subsubsection{Analysis on the quality of selection samples}
To highlight the improvement of our informative positive manifold coverage strategy, Figure \ref{fig:werding_kde} shows the KDE plot of the behavior of our approach compared with the CAADA method on two categories of the Office-Home dataset.

\begin{figure}[!htb]
    \centering
    \includegraphics[width=\textwidth]{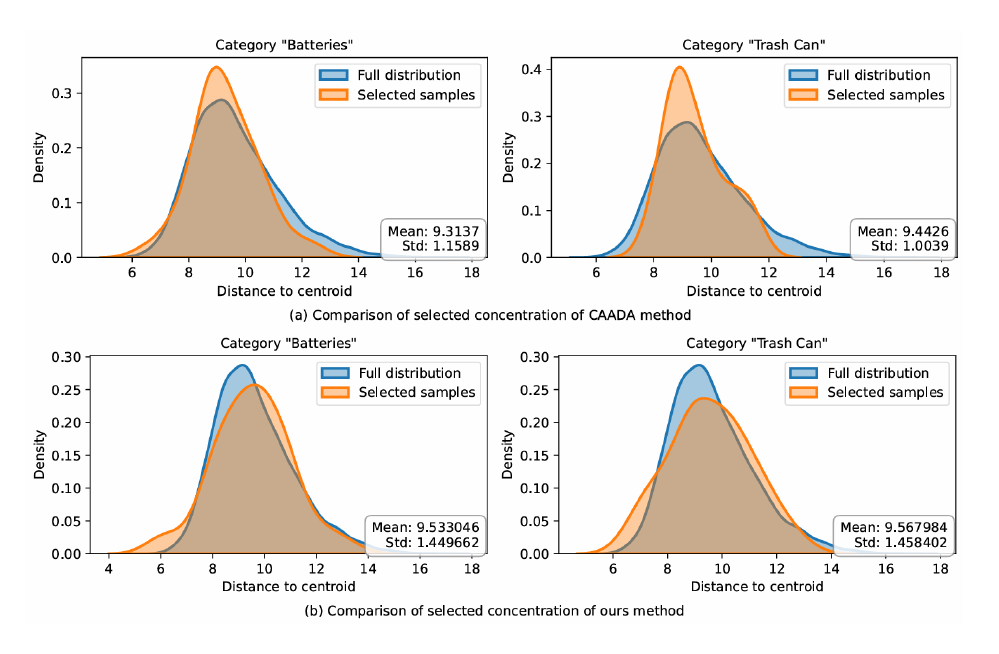}
    \caption{KDE plot for behavior of selection strategy of ours and CAADA in the $2nd$ active round with two categories, "Batteries" and "Trash Can", in the Ar$\rightarrow$Cl task of the OfficeHome dataset.}
    \label{fig:werding_kde}
\end{figure}

Consider the "Batteries" category, the two figures exhibit noticeably different sample selection behaviors.
As discussed in Section \ref{sec3}, the CAADA 
mainly focuses on the high-density center, while distant examples are being ignored.
In contrast, our method shows a substantially broader 
coverage that closely matches the distribution of the full category.
For the "Trash Can" category, the difference between the two methods is even more obvious. 
CAADA concentrates the central region and largely overlooks the distribution tails.
This observation suggests that our approach offers broader category coverage and more efficiently maintains category-specific diversity compared to the point-wise selection behavior.
            
\begin{figure}[!htb]
                \centering
                \includegraphics[width=\textwidth]{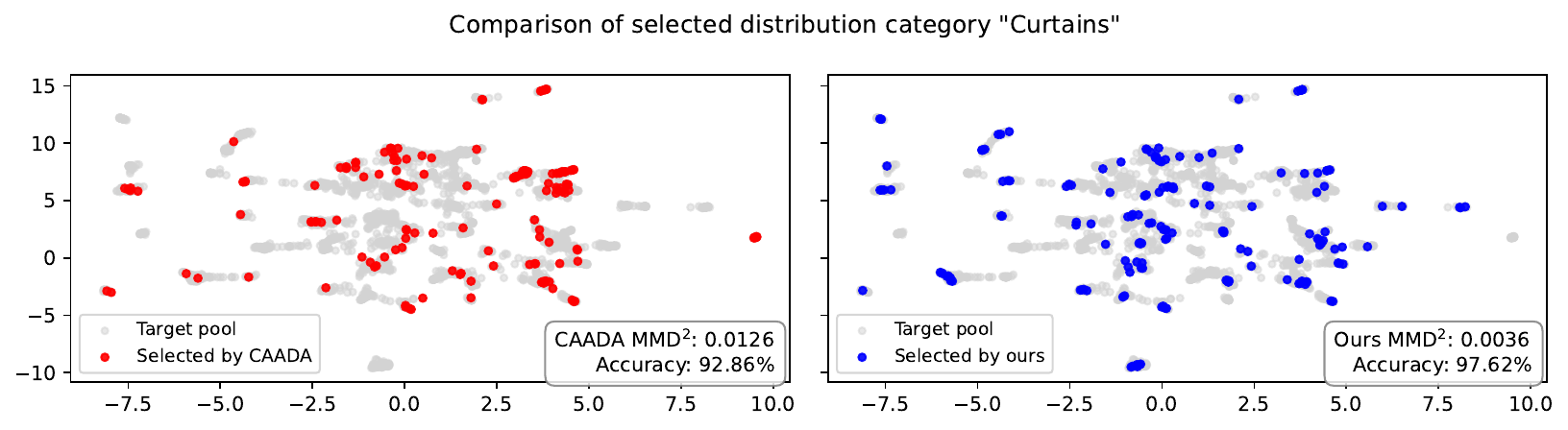}
                \caption{Comparison data distribution between CAADA and our selection method with the "Curtains" category in the Ar$\rightarrow$Cl task on the Office-Home dataset.}
                \label{fig:curtains}
            \end{figure} 
            
In Figure \ref{fig:curtains} we choose the "Curtains" category of the Office-Home benchmark to compare the quality of chosen samples.
Compared with CAADA, our method selects samples that are more uniformly distributed across the target manifold, 
reducing sample redundancy and improving category coverage.
This is reflected in the substantially lower MMD value, leading to a $4.76\%$ performance gain and $97.62\%$ accuracy.

\begin{figure}[!htb]
                \centering
                \includegraphics[width=\textwidth]{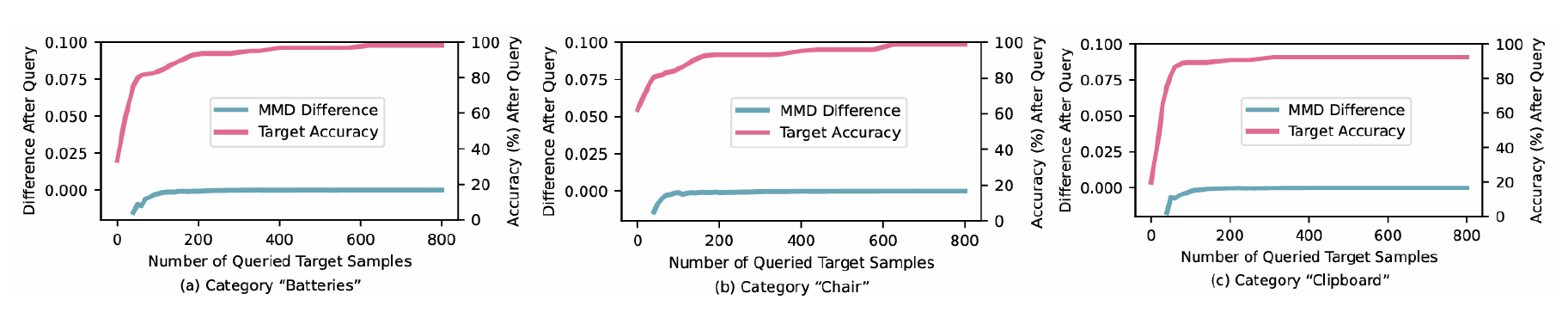}
                \caption{Classification loss on the validation set and predicting accuracy on the target domain in three different categories in the Ar$\rightarrow$Pr task on Office-Home dataset.}
                \label{fig:mmd_acc}
            \end{figure} 
            
We 
demonstrate the convergence of model performance as well as the MMD difference after each query round in Fig. \ref{fig:mmd_acc}.
In this experiment, we choose the budget $\mathcal B=10$ and run with $80$ active rounds.
The MMD difference varies only slightly after 5 rounds and remains nearly unchanged from round 10 onward, whereas the target accuracy steadily increases and eventually converges to its upper limit.

\begin{figure}[!htb]
                \centering
                \includegraphics[width=\textwidth]{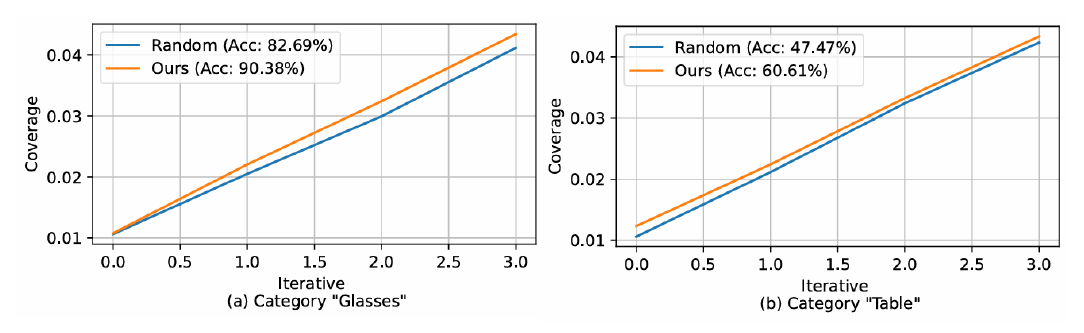}
                \caption{Comparision between the coverage of our proposed method with the random sampling method on the target domain in two different categories in the Ar$\rightarrow$Cl task on Office-Home dataset.}
                \label{fig:coverage_compare}
            \end{figure}

We further compare the coverage achieved by our proposed method with that of random sampling, where samples are selected independently and identically distributed. As shown in Fig. \ref{fig:coverage_compare}, the coverage gap gradually increases over the iterative sampling rounds, indicating that our method selects samples that provide broader coverage of the target category.
This improved coverage contributes to better performance of our proposed method.

\subsubsection{Analysis on contribution of each component}
Table \ref{tab:ablation_components} demonstrates that each component also contributes to the improvement of our model.
Compared with CAADA, our \textbf{Coverage Only} boosts the performance by $6.61\%$ while adding a weighted-region scheme further improves the accuracy by $14.15\%$ with a small standard deviation, about $0.05\%$.
Together, these two modules increase the performance by $20.76\%$ and reduce the deviation compared with the baseline CAADA.
Moreover, using coverage for sample selection significantly reduces the runtime, with \textbf{Coverage Only} and our method achieving runtime reductions of about $94.20\%$ and $87.94\%$, respectively, compared with CAADA.

\begin{table}[t]
\centering
\caption{Accuracy and runtime with different model components for domain adaptation task Rw$\rightarrow$Pr on Office-Home dataset with three different seed runs.}
\label{tab:ablation_components}
\begin{tabular}{lccccc}
\toprule
Model & Coverage & Region Wt. & Acc. (\%) & Time (s) \\
\midrule
CAADA & $\times$ & $\times$ & $74.22 \pm 2.36$ & 11407.54 \\
\midrule
Coverage Only & $\checkmark$ & $\times$ & 80.83 $ \pm $ 0.06 & 664.75 \\
Coverage + Region Wt. & $\checkmark$ & $\checkmark$ & 94.98 $ \pm $ 0.05 &  1376.05 \\

\bottomrule
\end{tabular}
\end{table}

\begin{figure}[!htb]
                \centering
                \includegraphics[width=\textwidth]{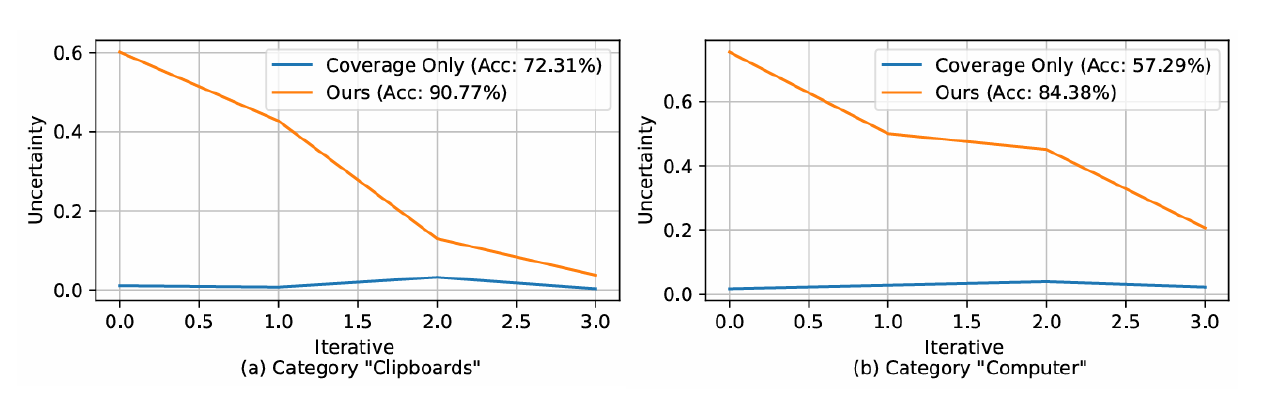}
                \caption{Comparision between the uncertainty of selected samples of our proposed method with the Coverage Only sampling method on the target domain in two different categories in the Rw$\rightarrow$Pr task on Office-Home dataset.}
                \label{fig:uncertainty_compare}
\end{figure}

To further investigate the $14.15\%$ improvement of uncertainty weighting over Coverage Only, we analyze the uncertainty of the samples selected by the two strategies in two different categories, \textit{Clipboard} and \textit{Computer}.
As shown in Fig. \ref{fig:uncertainty_compare}, our method consistently selects samples with substantially higher uncertainty, whereas Coverage Only predominantly selects samples with near-zero uncertainty.

\subsubsection{Analysis on the scoring mechanism}
\begin{wraptable}{r}{0.50\textwidth}
\vspace{-10pt}
\centering
\caption{Accuracy with different combinations of positive and negative manifolds for domain adaptation task Rw$\rightarrow$Pr on the Office-Home dataset.}
\label{tab:ablation_manifold}
\begin{tabular}{lcc}
\toprule
$\alpha$ & $\beta$ & Acc. (\%) \\
\midrule
$1$ & $0$ & 94.98 $ \pm $ 0.01 \\
\midrule
$1$ & $1$ & 94.90 $\pm$ 0.01 \\
$1$ & $-1$ & 94.97 $\pm$ 0.02 \\
$1$ & $-0.5$ & 94.98 $\pm$ 0.05 \\

\bottomrule
\end{tabular}
\vspace{-10pt}
\end{wraptable}

We additionally examine whether incorporating the negative-manifold term into Eq. \ref{eq:score} provides complementary information for target selection.
The detail of our combination can be formulated as follows:
\begin{equation*}
    s(x; f_c) = \alpha \cdot s(x, \mathcal M_c^{+}; f_c) + \beta \cdot \Delta h_c(x; \mathcal M_c^{-})
\end{equation*}
with $\Delta h_c(x; \mathcal M_c^{-})$ is defined as
\[
 \Delta h_c(x; \mathcal M_c^{-})  = \frac{1}{N_{\mathcal{T}_U}} \sum_{i=1}^{N_{\mathcal{T}_U}} \max_{\mathbf{x}' \in \{\mathcal{T}_L \setminus \mathcal M_c^{+}\} \cup \{x\}} h_c(x_i,x') - \frac{1}{N_{\mathcal{T}_U}} \sum_{i=1}^{N_{\mathcal{T}_U}} \max_{\mathbf{x}'\in \{\mathcal{T}_L \setminus \mathcal M_c^{+}\}} h_c(x_i,x')
\]


As shown in Table \ref{tab:ablation_manifold}, the resulting accuracy varies by only approximately $0.07\%$, with no consistent improvement observed when adjusting $\beta$. 
This suggests that the negative manifold provides only marginal improvement.
We therefore set $\beta=0$ and use only the positive manifold in all main experiments.

\section{Conclusion}\label{sec6}
In this paper, we introduced a category-aware active domain adaptation method that improves individual category performance while faithfully reproducing the target distribution of each category.
Motivated by the redundancy and parameter sensitivity of point-wise, influence-based selection, we reformulated the acquisition as a region-based problem and proposed an informative positive manifold coverage score.
By restricting the coverage measure to the positive manifold of the target class and approximating the coverage gain through Monte-Carlo estimation, our selection strategy queries samples that expand the represented region.
We further introduced an uncertainty guidance term that biases the coverage toward regions near the decision boundary, and we proved that this weighting preserves the positive coverage gain up to a multiplicative constant.
Our theoretical analysis explains the local-concentration behaviour of influence-only selection through Lipschitz continuity and formally establishes the redundancy-reduction property of our region-based objective.
Empirical results demonstrate that this method outperforms existing category-aware active domain adaptation baselines and other active domain adaptation methods across various benchmarks, achieving significant performance improvements.
Given the promising results, one possible extension is to include a joint fine-tuning of the extractor and the classifier.
Additionally, a finer sampling method can be used to balance the uncertainty-coverage trade-off in the scoring mechanism.

\backmatter

\begin{appendices}






\end{appendices}

\bibliography{sn-bibliography}

\end{document}